\documentclass{tlp}

\usepackage[table]{xcolor}
\usepackage[textsize=tiny]{todonotes}
\usepackage{csvsimple}
\usepackage{booktabs,siunitx}

\usepackage{url}
\usepackage{amsmath}
\usepackage{graphicx}
\usepackage{multirow}
\usepackage{tikz}
\usepackage{listings}
\newtheorem{definition}{Definition}
\newtheorem{theorem}{Theorem}
\newtheorem{example}{Example}

\definecolor{darkgreen}{RGB}{68,180,46}
\definecolor{darkgray}{RGB}{80,80,80}
\newcommand{\ucap}{\ensuremath{\mathit{UCAP}}}
\newcommand{\iucap}{\ensuremath{\mathit{IUCAP}}}

\usetikzlibrary{arrows,positioning,automata,shadows,fit,shapes,calc}
\usetikzlibrary{arrows.meta}
\usetikzlibrary{shapes.multipart}
\usetikzlibrary{decorations.text, shapes.misc}
\usetikzlibrary{decorations.pathmorphing}
\usetikzlibrary{shapes.geometric}
\usetikzlibrary{calc}

\newcommand{\stackk}[4]{
  \foreach \i in {1,...,#1} {
    \draw[fill=white!10] #2 ++({0.1*(#1)},{-0.1*(#1)}) ++({-0.1*\i},{0.1*\i}) rectangle +#3;
  }
  \path #2 -- +#3 node[pos=0.5] {#4};
}
\newcommand{\stack}[6]{
  \foreach \i in {1,...,#1} {
    \draw[fill=#5,draw=#6] #2 ++({0.1*(#1)},{-0.1*(#1)}) ++({-0.1*\i},{0.1*\i}) rectangle +#3;
  }
  \path #2 -- +#3 node[pos=0.5,text=#6] {#4};
}

\definecolor{aauBlue}{RGB}{85,158,195}
\definecolor{aauBlue2}{RGB}{67,111,129}
\newcommand{\dnot}{\ensuremath{\mathit{not}\,}}

\begin{document}

\lefttitle{A. Tarzariol, M. Gebser, K. Schekotihin, and M. Law}

\jnlPage{1}{16}
\jnlDoiYr{2021}
\doival{10.1017/xxxxx}

\title[Efficient lifting of SBCs for complex combinatorial problems]{Efficient lifting of symmetry breaking constraints for complex combinatorial problems
}

\begin{authgrp}
\author{\sn{Alice} \gn{Tarzariol}}
\affiliation{University of Klagenfurt, Austria}
\author{\sn{Martin} \gn{Gebser}}
\affiliation{University of Klagenfurt, Austria \\Graz University of Technology, Austria}
\author{\sn{Mark} \gn{Law}}
\affiliation{Imperial College London, UK}
\author{\sn{Konstantin} \gn{Schekotihin}}
\affiliation{University of Klagenfurt, Austria}
\end{authgrp}

\history{\sub{xx xx xxxx;} \rev{xx xx xxxx;} \acc{xx xx xxxx}}

\maketitle

\begin{abstract}

Many industrial applications require finding solutions to challenging combinatorial problems. Efficient elimination of symmetric solution candidates is one of the key enablers for high-performance solving. However, existing model-based approaches for symmetry breaking are limited to problems for which a set of representative and easily-solvable instances is available, which is often not the case in practical applications. 
This work extends the learning framework and implementation of a model-based approach for Answer Set Programming to overcome these limitations and address challenging problems, such as the Partner Units Problem. 
In particular, we incorporate a new conflict analysis algorithm in the Inductive Logic Programming system \textsc{ILASP}, redefine the learning task, and suggest a new example generation method to scale up the approach.
The experiments conducted for different kinds of Partner Units Problem instances demonstrate the applicability of our approach and the computational benefits due to the first-order constraints learned.

\textbf{Under consideration for acceptance in
Theory and Practice of Logic Programming (TPLP).
}
\end{abstract}


\section{Introduction}\label{sec:introduction}


Finding solutions to hard combinatorial problems is important for various applications, including configuration, scheduling, or planning. Modern declarative solving approaches allow programmers to easily encode various problems and then use domain-independent solvers to find solutions for given instances. The search performance depends highly on both the encoding quality and selected solver parameters. The latter issue can partially be solved using portfolio solvers that use machine learning to select the best possible parametrization of underlying solving algorithms \citep{portfolio}. However, writing an optimal encoding remains a challenge that requires an experienced programmer to clearly understand the problem up to details, such as possible structures present in instances, their invariants, and symmetries. 

Automatic computation of Symmetry Breaking Constraints (SBCs) can greatly simplify the programming task by extending a given encoding with constraints eliminating symmetric solution candidates, i.e., a set of candidates where each one can be obtained from another by renaming constants \citep{margot07a,DBLP:conf/sat/KatebiSM10,walsh12a}. Existing approaches compute SBCs either for a first-order problem encoding or for only one instance of a given problem. The latter -- \textit{instance-specific} methods -- compute SBCs online, i.e., before or during each invocation of a solver \citep{DBLP:conf/cp/Puget05,cojejepesm06a,drtiwa11a}. As a result, the obtained SBCs are not transferable to other instances and their computation might significantly increase the total solving time since the problem of finding SBCs is intractable. The \textit{model-based} methods are usually applied offline and aim at finding general SBCs breaking symmetries of a class of instances. However, these methods are either limited to local symmetries occurring due to definitions of variable domains \citep{debobrde16a} or require representative sets of instances \citep{DBLP:conf/cpaior/MearsBWD08,tagesc21a} to identify SBCs that highly likely eliminate symmetries for all instances of this class. Roughly, model-based approaches apply graph-based methods, e.g., \textsc{saucy} \citep{saucy}, to find candidate symmetries of the given instances and then lift the obtained information to first-order SBCs.

The main challenge in the application of model-based approaches is that they must be able to access or generate instances that 
\begin{enumerate*}[label=\emph{(\roman*)}]
  \item comprise symmetries representative for the whole instance distribution, and 
  \item are simple enough to allow the implementation to compute all of their solutions.
\end{enumerate*}
Consider a small example of the well-studied \textit{Partner Unit Problem} (PUP) \citep{DBLP:conf/cpaior/AschingerDFGJRT11,DBLP:journals/jcss/TeppanFG16}, which is an abstract representation of  configuration problems occurring in railway safety or building security systems. As shown in Fig.\ \ref{fig:pup-ex}a, the input of the problem is given by a set of units $U$ and a bipartite graph $G=(S,Z,E)$, where $S$ is a set of sensors, $Z$ is a set of security/safety zones, and $E$ is a relation between $S$ and $Z$. The task is to find a partition of vertices $v \in S\cup Z$ into bags $u_i \in U$, presented in Fig.\ \ref{fig:pup-ex}b, such that the following requirements hold for each bag:
\begin{enumerate*}[label=\emph{(\roman*)}] 
  \item the bag contains at most $\ucap$ many sensors and $\ucap$ many zones; and
  \item the bag has at most $\iucap$ adjacent bags, where the bags $u_1$ and $u_2$ are adjacent whenever $v_i \in u_1$ and $v_j \in u_2$ for some $(v_i, v_j) \in E$.
\end{enumerate*}
The given example shows the smallest instance representing a class of building security systems named \textit{double} by \cite{DBLP:conf/cpaior/AschingerDFGJRT11}. 
Despite being the simplest instance, it has $145368$ solutions, $98.9\%$ of which can be identified as symmetric (for instance, by renaming the units of a solution). Therefore, the enumeration of symmetries for PUP instances is problematic, even for the smallest and simplest ones.
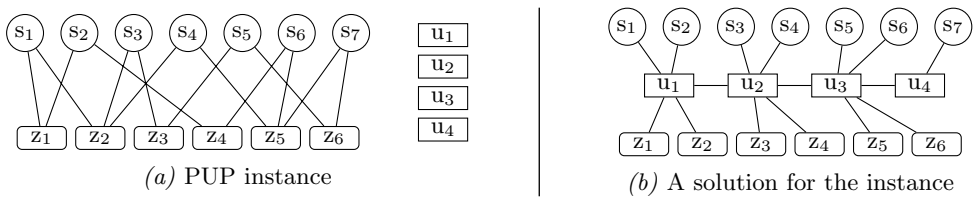
\begin{figure}[b]
  \centering
%
%
%
%

\tikzstyle{sensor} = [draw, circle, inner sep=0em, minimum width=1.5em]
\tikzstyle{zone} =   [draw, rectangle, inner sep=.2em, minimum width=2em, rounded corners=0.2em]
\tikzstyle{unit} =   [draw, rectangle, inner sep=.2em, minimum width=2em]

\begin{minipage}{.5\textwidth}
\centering
\begin{tikzpicture}[x=1.7em,y=1.7em,baseline=0pt]
\foreach \x in {0,...,6}
{
    \pgfmathtruncatemacro{\y}{\x+1}
    \node[sensor] (s\y) at (1.3*\x,1.25) {$\mathrm{s_{\y}}$};
}

\foreach \x in {0,...,5}
{
    \pgfmathtruncatemacro{\y}{\x+1}
    \node[zone] (z\y) at (.4+1.4*\x,-1.25) {$\mathrm{z_\y}$};
}

\pgfmathsetmacro{\yp}{1.2}
\foreach \x in {1,...,4}
{
    \pgfmathsetmacro{\y}{\yp-.75}
    \node[unit] (u\y) at (10,\yp) {$\mathrm{u_\x}$};
    \global\let\yp=\y
}

\draw 
(z1)  -- (s1)
(z1)  -- (s2)
(z2)  -- (s1)
(z2)  -- (s4)
(z2)  -- (s3)
(z3)  -- (s3)
(z3)  -- (s5)
(z4)  -- (s2)
(z4)  -- (s6)
(z5)  -- (s6)
(z5)  -- (s4)
(z5)  -- (s7)
(z6)  -- (s7)
(z6)  -- (s5);
\end{tikzpicture}%
\vspace{5pt}
\\ \emph{(a)} PUP instance
\end{minipage}%
\hfill\vline\hfill
\begin{minipage}{.4\textwidth}
\centering
\begin{tikzpicture}[x=1.7em,y=1.7em,baseline=0pt]

\foreach \x in {0,...,6}
{
    \pgfmathtruncatemacro{\y}{\x+1}
    \node[sensor] (s\y) at (1.3*\x,1.4) {$\mathrm{s_{\y}}$};
}

\foreach \x in {0,...,5}
{
    \pgfmathtruncatemacro{\y}{\x+1}
    \node[zone] (z\y) at (.4+1.4*\x,-1.4) {$\mathrm{z_\y}$};
}

\foreach \x in {0,...,3}
{
    \pgfmathtruncatemacro{\y}{\x+1}
    \node[unit] (u\y) at (1+2*\x,0) {$\mathrm{u_\y}$};
}

\draw 
(z1) -- (u1) -- (s1)
(z2) -- (u1) -- (s2)
(z3) -- (u2) -- (s3)
(z4) -- (u2) -- (s4)
(z5) -- (u3) -- (s5)
(s7) -- (u4) 
(z6) -- (u3) -- (s6);

\draw (u1) -- (u2) -- (u3) -- (u4);
\end{tikzpicture}\vspace{5pt}
\\ 
\centering \emph{(b)} A solution for the instance
\end{minipage}
  \caption{Partner Unit Problem example with $\ucap=\iucap=2$ \label{fig:pup-ex}}
\end{figure}

The \textit{model-based} method suggested by \cite{tagesc21a} applies an \textit{instance-specific} tool for identifying symmetries of small but representative instances of a target distribution, and then generalizes respective examples by means of Inductive Logic Programming (ILP).
To lift the symmetries, the method needs to investigate all solutions of the analyzed instances, making it inapplicable if no trivial satisfiable instances exist. In this work, we address such limitations and extend the framework's applicability to combinatorial problems lacking trivial representative instances. 
In particular, our paper makes the following contributions:
\begin{itemize}
  \item We propose a new definition of the ILP learning task and a corresponding implementation for the input generation, which allow the approach to scale with respect to the number of answer sets of the analyzed instances, thus, learning efficient first-order constraints. 
  \item We provide a novel conflict analysis method for the learning system \textsc{ilasp} \citep{larubr20b} that significantly improves the efficiency of the constraint learning.
  \item We present an extensive experimental study conducted on three kinds of PUP benchmarks shows that the new method clearly outperforms the legacy approach in terms of learning and solving performance. 
\end{itemize}

The paper is organized as follows: after a short introduction of preliminaries in Section~\ref{sec:background}, we introduce the revised learning framework and new conflict analysis method in Section~\ref{sec:method}. Results of our experimental study are presented in Section~\ref{sec:experiments}, followed by conclusions and future work in Section~\ref{sec:conclusions}.

\section{Background}\label{sec:background}
In this section, we briefly introduce basic notions of Answer Set Programming, Inductive Logic Programming, and the work by \cite{tagesc21a}, which uses ILP to learn first-order constraints from symmetries of ground ASP programs. 

\subsection{Answer Set Programming}\label{subsec:asp}
Answer Set Programming (ASP) is a declarative programming paradigm based on non-monotonic reasoning and the stable model semantics \citep{gellif91a}. 
Over the past decades, ASP has attracted considerable interest thanks to its elegant syntax, expressiveness, and efficient system implementations, 
successfully adopted in numerous domains like, e.g., configuration, robotics, or biomedical applications \citep{ergele16a,fafrsctate18a}. 
We briefly present the syntax and semantics of ASP, and refer the reader to textbooks
\citep{gekakasc12a,lifschitz19a} for more in-depth 
introductions.

\paragraph{Syntax.} 
An ASP program $P$ is a set of \textit{rules}~$r$ of the form: 
\begin{equation*}
  a_0 \gets a_1, \dots, a_m, \dnot a_{m+1}, \dots, \dnot a_n
\end{equation*}
where $\mathit{not}$ stands for \textit{default negation} and $a_i$, for $0\leq i \leq n$, are atoms. 
An \textit{atom} is an expression of the form $p(\overline{t})$, where $p$ is a predicate, $\overline{t}$ is a possibly empty vector of terms, and the predicate $\bot$ (with an empty vector of terms) represents the constant \textit{false}.
Each \textit{term} $t$ in $\overline{t}$ is either a variable or a constant.
A \textit{literal} $l$ is 
an atom $a_i$ (positive) or its negation $\dnot a_i$ (negative). 
The atom $a_0$ is the \textit{head} of a rule~$r$,
denoted by $H(r)=a_0$,
and the \textit{body} of~$r$ includes the positive or negative, respectively,
body atoms $B^+(r) = \{a_1, \dots, a_m\}$ and $B^-(r) = \{a_{m+1}, \dots, a_n\}$.
%
A rule~$r$ is called a \textit{fact} if $B^+(r)\cup B^-(r)=\emptyset$, and a \textit{constraint} if $H(r) = \bot$.

\paragraph{Semantics.}
The semantics of an ASP program $P$ is given in terms of its \textit{ground instantiation} $P_{\mathit{grd}}$, which is obtained by
replacing each rule $r\in\nolinebreak P$ with its instances obtained by substituting the variables in $r$ by constants occurring in $P$.
Then, an \textit{interpretation} $\mathcal{I}$ is a set of (\textit{true}) ground atoms occurring in $P_{\mathit{grd}}$ that does not contain~$\bot$. 
An interpretation $\mathcal{I}$ \textit{satisfies} a rule $r \in P_{\mathit{grd}}$ if $B^+(r) \subseteq \mathcal{I}$ and $B^-(r) \cap \mathcal{I}=\emptyset$ imply $H(r) \in \mathcal{I}$, and
$\mathcal{I}$ is a \textit{model} of $P$ if it satisfies all rules $r\in P_{\mathit{grd}}$. 
A model $\mathcal{I}$ of~$P$ is \textit{stable} if it is a subset-minimal model of the reduct $\{H(r) \gets B^+(r) \mid r \in P_{\mathit{grd}}, B^-(r) \cap \mathcal{I} = \emptyset\}$, and we denote the set of all stable models, also called answer sets, of $P$ by $\mathit{AS}(P)$.

\subsection{Inductive Logic Programming}\label{subsec:ilp}
Inductive Logic Programming (ILP) is a form of machine learning whose goal is to learn a logic program that explains a set of observations in the context of some pre-existing knowledge \citep{crdumu20a}.
Since its foundation, the majority of research in the field has addressed Prolog semantics 
although applications in other paradigms appeared in the last years.  
The most expressive ILP system for ASP is \textit{Inductive Learning of Answer Set Programs} (\textsc{ilasp}), which can solve a variety of ILP tasks \citep{larubr14a,ilasp}.

\paragraph{Learning from Answer Sets.}
A \textit{learning task} for \textsc{ilasp} is given by a triple $\langle B,E,H_M \rangle$, where an ASP program $B$ defines the \textit{background knowledge}, the set $E$ comprises two disjoint subsets $E^+$ and $E^-$ of \textit{positive} and \textit{negative examples}, and the \textit{hypothesis space} $H_M$ is defined by a language bias $M$, which limits the potentially learnable rules \citep{larubr14a}.
%
Each example $e \in E$ is a pair $\langle e_{\mathit{pi}}, C\rangle$ called \textit{Context Dependent Partial Interpretation} (CDPI), where 
\begin{enumerate*}[label=\emph{(\roman*)}]
  \item $e_{\mathit{pi}}$ is a \textit{Partial Interpretation} (PI) defined as pair of sets of atoms $\langle T, F\rangle$, called \textit{inclusions} ($T$) and  \textit{exclusions} ($F$), respectively, and
  \item $C$ is an ASP program defining the \textit{context} of PI $e_{\mathit{pi}}$.
\end{enumerate*}
Given a (total) interpretation $\mathcal{I}$ of a program $P$ and a PI $e_{\mathit{pi}}$, we say that $\mathcal{I}$ \textit{extends} $e_{\mathit{pi}}$ if $ T \subseteq \mathcal{I}$ and $F \cap \mathcal{I} = \emptyset$.
Given an ASP program $P$, an interpretation $\mathcal{I}$, and a CDPI $e=\langle e_{\mathit{pi}}, C\rangle$, we say that  $\mathcal{I}$ is an \textit{accepting answer set} of $e$ with respect to $P$ if $\mathcal{I} \in \mathit{AS}(P \cup C)$ such that $\mathcal{I}$ extends $e_{\mathit{pi}}$.

Each hypothesis $H \subseteq H_M$ learned by \textsc{ilasp} must respect the following criteria: 
\begin{enumerate*}[label=\emph{(\roman*)}]
  \item for each positive example $e \in E^+$, there is some accepting answer set of $e$ with respect to $B\cup H$; and 
  \item for any negative example $e \in E^-$, there is no accepting answer set of $e$ with respect to $B\cup H$.
\end{enumerate*} 
If multiple hypotheses satisfy the conditions, the system returns one of those with the lowest cost. By default, the cost $c_{r}$ of each rule $r \in H_M$ corresponds to its number of literals \citep{larubr14a}; however, the user can define a custom scoring function for defining the rule costs.
\cite{larubr18b} extend the expressiveness of \textsc{ilasp} by allowing noisy examples.
With this setting, if an example $e$ is not covered, i.e., there is an accepting answer set for $e$ if it is negative, or none if $e$ is positive, the corresponding weight 
is counted as a penalty. If no dedicated weight is specified, the example's weight is infinite, thus forcing the system to cover the example. Therefore, the learning task becomes an optimization problem with two goals: minimize the cost of $H$ and minimize the total penalties for the uncovered examples.

The \textit{language bias} $M$ for the \textsc{ilasp} learning task is specified using \textit{mode declarations}. Constraint learning, i.e., when the search space exclusively consists of rules $r$ with $H(r) = \bot$, requires only mode declarations for the body of a rule:
\texttt{\#modeb$($R,P,$($E$))$.} 
In this definition, the optional element \texttt{R} is a positive integer, called \textit{recall}, which sets the upper bound on the number of mode declaration applications in each rule.
\texttt{P} is a ground atom whose arguments are placeholders of type \texttt{var(t)} for some constant term \texttt{t}. 
In the learned rules, the placeholders will be replaced by variables of type~\texttt{t}. For each rule, there are at most $V_{\mathit{max}}$ variables and $B_{\mathit{max}}$ literals in the body, which are both equal to $3$ by default.
Finally, \texttt{E} is an optional modifier that restricts the hypothesis space further, limited in our paper to the \texttt{anti\_reflexive} and \texttt{symmetric} options that both work with predicates of arity~$2$. When using the former, the atoms of the predicate \texttt{P} should be generated with two distinguished argument values, while
rules generated with the latter take into account that the predicate \texttt{P} is symmetric.

In a constraint learning task, just as in other ILP applications \citep{cropdum20a}, the language bias must be defined manually for each ASP program $P$. A careful selection of the bias is essential since a too weak bias might not provide enough limitations for a learner to converge. In contrast, a too strong bias may exclude solutions from the search space, thus resulting in sub-optimal learned constraints. 

\paragraph{Conflict Driven Inductive Logic Programming (CDILP).} \label{subsec:cdilp}
Several \textsc{ilasp} releases have been developed in the last years, extending its learning expressiveness and applying more efficient search techniques \citep{larubr20b}.
%
Recently, \cite{law21a} introduced CDILP -- a new search approach that overcomes the limitation of previous \textsc{ilasp} versions regarding scalability with respect to the number of examples and further aims at efficiently addressing tasks with noisy examples. 
%
The approach exploits a set $\mathit{CC}$ of \textit{coverage constraints}, each defined by a pair $\langle e, F \rangle$, where $e \in E$ is a CDPI and $F$ is a propositional formula 
over identifiers for the rules in $H_M$. The formula is defined such that, for any $H \subseteq H_M$, if $H$ does not respect $F$, 
then $H$ does not cover $e$. 
CDILP interleaves the search for an optimal hypothesis $H$ (for the current
$\mathit{CC}$) with a ``conflict analysis'' phase. In this phase,
\textsc{ilasp} identifies (at least) one example $e$ not covered by $H$, which
was not determined by the current $\mathit{CC}$; then, it creates a new
conflict for $e$ and adds it to $\mathit{CC}$. If such an example does not
exist, \textsc{ilasp} returns the current $H$ as an optimal solution.
Otherwise, the system repeats the procedure with the updated set of conflicts. 
Using the Python interface, PyLASP, one can apply different conflict analysis
methods as long as they are proven to be \textit{valid}, i.e., a method must
terminate and compute formulas $F$ such that the current hypothesis $H$ does
not respect them. This requirement guarantees that the CDILP procedure
terminates and returns an optimal hypothesis for the learning task.

There are currently three built-in methods for conflict analysis in
\textsc{ilasp}, denoted by $\alpha$, $\beta$, and $\gamma$, each of which determines a
coverage constraint for an example $e$ that is not covered by a given hypothesis~$H$.
In the most stringent case, $\gamma$, \textsc{ilasp} computes a coverage constraint
that is satisfied by exactly those hypotheses that cover $e$. Identifying such a
comprehensive coverage formula has the advantage that any example will be
analyzed in at most one iteration, so that \textsc{ilasp} with $\gamma$ for
conflict analysis usually requires a small number of iterations only. On the other
hand, computing such a precise coverage formula is complex, meaning that an
iteration can take long time. For this reason, $\alpha$ and $\beta$ were
introduced. Both methods yield smaller formulas that can be computed in less
time. While this can lead to more iterations of CDILP, 
in some
domains, the overall runtime benefits from significantly shorter iterations.
However, our preliminary investigations showed that, for highly combinatorial PUP instances, even $\alpha$ and $\beta$ struggle to compute coverage constraints in acceptable time.
In this work, we thus introduce a
new conflict analysis method that brings significant improvements over 
$\alpha$, $\beta$, and $\gamma$ on PUP instances.



\subsection{Lifting SBCs for ASP}\label{subsec:framework}

\cite{tagesc21a} presented an approach to lift ground 
SBCs for ASP programs using ILP. Their system takes four kinds of inputs:
\begin{enumerate*}[label=\emph{(\roman*)}]
  \item an ASP program $P$ modeling a combinatorial problem; 
  \item two sets $S$ and $\mathit{Gen}$ of small satisfiable instances representative for a practical problem solved using $P$; 
  \item the hypothesis space $H_M$; and 
  \item the Active Background Knowledge $\mathit{ABK}$ as an ASP program 
comprising auxiliary predicate definitions and constraints learned so far.
\end{enumerate*} 
The \textit{generalization set} $\mathit{Gen}$ contains instances 
used to generate positive examples 
that the set of learned constraints must preserve. 
As a result, we increase the likelihood for the learned constraints to generalize beyond the training examples.
The instances of the \textit{training set} $S$ are passed to the instance-specific symmetry breaking system \textsc{sbass} \citep{drtiwa11a} to identify the symmetries of each instance in~$S$.
The output of \textsc{sbass}, $\Pi$, is a set of permutation group \textit{generators} (also called \textit{permutations}) subsuming groups of symmetric answer sets for the analyzed ground program. 
The framework by \cite{tagesc21a} uses this information to define the 
positive and negative examples for an ILP task and applies \textsc{ilasp} to solve it. 
The negative examples are associated with a weight, as the system aims at 
constraints that remove as many symmetric answer sets as possible but does not require eliminating all of them.


In their subsequent work, \cite{tagesc22} discuss four 
approaches for creating training examples with \textsc{sbass} for ground programs $P_i$, obtained by grounding $P$ with the instances $i \in S$. 
In particular, \textit{enum} enumerates all answer sets of $P_i$ and classifies each solution as a positive or negative example, according to the 
common lex-leader approach.
That is, if an answer set $\mathcal{I}$ can be mapped to a lexicographically smaller, symmetric answer set using the permutations $\Pi$, i.e., if $\mathcal{I}$ is dominated, it will produce a negative example. Otherwise, $\mathcal{I}$ yields a positive example.
In both cases, the inclusions are $\mathcal{I} \cap atoms(\Pi)$, 
where $atoms(\Pi)$ denotes the set of atoms occurring in~$\Pi$,
the exclusions are $atoms(\Pi) \setminus \mathcal{I}$, and the context is $i$.
%
%
%
On the other hand, the \textit{fullSBCs} approach exploits the \textsc{clingo} API to interleave the solving phase, which returns a candidate answer set $\mathcal{I}$, with the analysis of all its symmetric solutions. Thanks to the properties of permutation groups \citep{sakallah09a}, 
\textit{fullSBCs} can determine all symmetric answer sets by repeatedly applying the 
permutations~$\Pi$ 
to $\mathcal{I}$ until no new solutions can be obtained. 
This approach leads to 
a partition of the answer sets for an 
instance, where every partition cell consists of symmetric solutions.
For each obtained cell, the system labels the smallest answer set as a positive example and all 
remaining ones as negative examples, 
thus achieving full rather than partial symmetry breaking, while 
the \textit{enum} approach yields the latter only.

\cite{tagesc22} evaluate the performance of their methods on three versions of the pigeon-hole problem and the house-configuration problem
\citep{DBLP:conf/confws/FriedrichRFHSS11}. 
The $\mathit{ABK}$ they use contains predicates emulating arithmetic built-ins, and the search space is split to apply the learning framework iteratively and thus increase the learning efficiency. 
Given that the considered problem instances are defined in terms of unary predicates and those in the training set~$S$ have a small number of solutions (from about a dozen up to a few hundred), the suggested formulation of the ILP task admits a fast learning of first-order constraints speeding up the solving of (unsatisfiable) instances.
However, instances of complex application problems lack the presupposed  characteristics, rendering the previously proposed approaches inapplicable
and calling for a more scalable handling of training instances.


\section{Method}\label{sec:method}

This section presents an alternative version of the framework introduced by \cite{tagesc21a}, extending its applicability. 
First, we propose a revised ILP learning task and procedures necessary to define inputs of this task for difficult combinatorial problems, as illustrated in Fig.~\ref{fig:pipeline}.
Then, we describe a new conflict analysis method for the \textsc{ilasp} system enabling efficient constraint learning to handle this revised task.

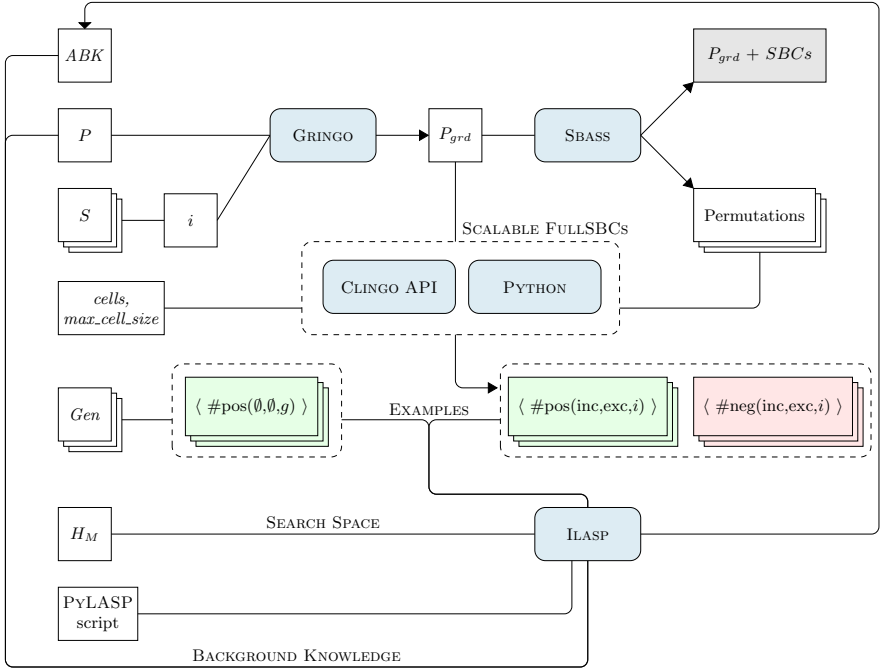
\begin{figure}[t]\hspace*{-1.25cm}
  {\small
\scalebox{0.7}{\begin{tikzpicture}[>=triangle 60]

 \draw  [draw=white] (-8.5,1) rectangle ++(1,1) ;
 \draw [draw=black]  (-6,2.5)  rectangle ++(1,1) node  [midway,text width=3.4cm,align=center] {$\mathit{ABK}$};
\draw[-,rounded corners] (-6,3)  --  (-7,3) -- (-7,-8.5) -- (4,-8.5) -- (4,-6.5);

 \draw [draw=black]  (-6,1)  rectangle ++(1,1) node  [midway,text width=3.4cm,align=center] {$P$};
 \draw[- ] (-5,1.5) -- (-2,1.5);
 \draw[-,rounded corners] (-6,1.5)  --  (-7,1.5) -- (-7,-8.5) --  node [above,midway,align=center] { \textsc{Background Knowledge}} (4,-8.5) -- (4,-6.5);

   \stackk{3}{ (-6,-0.5)}{(1,1)}{$S$};
   \draw[- ] (-4.8,-0.1) -- (-4,-0.1);

 \draw[-,rounded corners] (-5,-1.75)  --  (-1,-1.75);
 \draw  [draw=black,fill=white]  (-6,-2.25)  rectangle ++(2,1) node  [midway,text width=2cm,align=center] {\textit{cells, max\_cell\_size}};

 \stackk{3}{ (-6,-4.25)}{(1,1)}{$\mathit{Gen}$};
 \draw[-,rounded corners] (-4.8,-3.85)  --  (1,-3.85) --  (1,-5.25) --  (4,-5.25) --  (4,-5.5);


   \draw  [draw=black]  (-6,-6.5)  rectangle ++(1,1) node  [midway,text width=3.4cm,align=center] {$H_M$};
   \draw[-,rounded corners] (-5,-6)  -- node [above,midway,align=center] { \textsc{Search Space}} (3,-6);

 \draw[-,rounded corners] (-5,-7.5)  -- (3.7,-7.5)  -- (3.7,-5.5);
 \draw  [draw=black,fill=white]  (-6,-8)  rectangle ++(1.5,1) node  [midway,text width=1.4cm,align=center] {\textsc{PyLASP} script};

 \draw [draw=black] (-4,-.6) rectangle ++(1,1) node  [midway,text width=4cm,align=center] {$i$};
 \draw[- ] (-3,-0.1) -- (-2,1.5);
 
 \draw [rounded corners=0.2cm,fill=aauBlue!20] (-2,1) rectangle ++(2,1) node [midway,text width=3.4cm,align=center] {\textsc{Gringo}};
 \draw[-> ] (0,1.5) -- (1,1.5);
 
 \draw [draw=black] (1,1) rectangle ++(1,1) node  [midway,text width=4cm,align=center] {$P_{grd}$};
 \draw[- ] (2,1.5) -- (3,1.5);
 \draw[->, rounded corners] (1.5,1) --  (1.5,-3.25) --  (2.3,-3.25) ;
 
 \draw [rounded corners=0.2cm,fill=aauBlue!20] (3,1) rectangle ++(2,1) node [midway,text width=3.4cm,align=center] {\textsc{Sbass}};
 \draw[-> ] (5,1.5) -- (6,0.5);
 \draw[-> ] (5,1.5) -- (6,2.5);
 
 \draw [draw=black,fill=black!10]  (6,2.5)  rectangle ++(2.5,1) node  [midway,text width=3.4cm,align=center] {$P_{grd}$ + $SBCs$};
 
     \draw[-,rounded corners] (7.25,-0.5)  -- (7.25,-1.75)  --  (4,-1.75);
 \stackk{3}{ (6,-0.5)}{(2.3,1)}{Permutations};
 
 \draw [rounded corners=0.2cm,inner sep=0pt,dashed,fill=white!20] (-1.4,-2.25) rectangle ++(6,1.75)
 node at (3.2,-0.25)  {\textsc{Scalable FullSBCs}};
 \draw [rounded corners=0.2cm,fill=aauBlue!20] (-1,-1.85) rectangle ++(2.5,1) node [midway,text width=3.4cm,align=center] {\textsc{Clingo API}};
  \draw [rounded corners=0.2cm,fill=aauBlue!20] (1.75,-1.85) rectangle ++(2.5,1) node [midway,text width=3.4cm,align=center] {\textsc{Python}};

 


 \draw[-,rounded corners] (2.75,-3.85)  --  (1,-3.85) --  (1,-5.25) --  (4,-5.25) --  (4,-5.5);
 \draw [rounded corners=0.2cm,inner sep=0pt,dashed,fill=white!20] (2.35,-4.55) rectangle ++(7,1.75);
 \stack{3}{ (6,-4.15)}{(3,1.1)}{$\langle $ \#neg(inc,exc,$i$) $\rangle$}{red!10}{black}; 
 \stack{3}{ (2.5,-4.15)}{(3,1.1)}{$\langle $ \#pos(inc,exc,$i$) $\rangle$}{green!10}{black};

 \draw [rounded corners=0.2cm,inner sep=0pt,dashed,fill=white!20] (-3.85,-4.55) rectangle ++(3.2,1.75) ;
 \stack{3}{ (-3.6,-4.15)}{(2.5,1.1)}{$\langle$ \#pos($\emptyset$,$\emptyset$,$g$) $\rangle$}{green!10}{black}; 

   \draw [rounded corners=0.2cm,fill=aauBlue!20] (3,-6.5) rectangle ++(2,1) node [midway,text width=3.4cm,align=center] {\textsc{Ilasp}};
   \draw[black,rounded corners, arrows={-Triangle[angle=90:5pt,black,fill=black]}] (5,-6) -- (9.5,-6) -- (9.5,4) -- (-5.5,4) -- (-5.5,3.5);
   \node[align=center] at (1,-3.65){ \textsc{Examples}};

 \end{tikzpicture}
 }}
 \caption{Revised learning framework implementation}
 \label{fig:pipeline}
 \end{figure}

\subsection{Revised ILP Task}\label{subsec:ilpTask}
\paragraph{Training Examples.}
\cite{tagesc21a} propose approaches to the example generation for an ILP learning task, which yield a number of examples proportional to the number of solutions for each problem instance in $S$. However, for difficult problems with many symmetries, even the simplest instances might yield a large number of answer sets. Their enumeration might thus take unacceptably long time. Moreover, even if the enumeration succeeds, the number of obtained examples is often too large to be handled by \textsc{ilasp}.
To overcome these issues, we propose two scalable approaches to the generation of examples for each instance in $S$ based on the \textit{enum} and \textit{fullSBCs} strategies by \cite{tagesc21a}.
The \textbf{scalable enum} strategy generates examples from a portion of solutions, which are sampled from at most $n$ random answer sets.
For each candidate solution, the lex-leader criterion is applied to determine whether another symmetric answer set dominates it.
This approach does not guarantee a fixed ratio between positive and negative examples, and it might fail to identify symmetric answer sets, as inspecting single applications of permutation group generators does not achieve full symmetry breaking \citep{sakallah09a}. Nevertheless, the fixed number of considered answer sets provides means to limit the time required for the definition of a learning task. 
The \textbf{scalable fullSBCs} approach 
creates a set of examples configured by two parameters: 
\begin{enumerate*}[label=\emph{(\roman*)}]
  \item \textit{cells} defines the number of cells of symmetric solutions to analyze; and 
  \item \textit{max\_cell\_size} limits the maximal number of negative examples generated per cell.
\end{enumerate*}
For each cell, the approach 
adds the first \textit{max\_cell\_size} symmetric answer sets as negative examples. Next, it explores the whole cell of symmetric solutions and takes the smallest one as a positive example. 
%
As a result, this method generates a controlled number of positive and negative examples, regardless of how many solutions there may be for a given instance. In fact, at most \textit{cells} many positive and
$\text{\textit{cells}} \times \text{\textit{max\_cell\_size}}$ many negative examples can be obtained in total.%
%
%

\paragraph{Background Knowledge.}
Learning SBCs that improve the grounding and solving efficiency is crucial for difficult combinatorial problems.
%
%
Therefore, the background knowledge $\mathit{ABK}$ should provide necessary auxiliary predicates that allow an ILP system to incorporate such constraints in the search space.
Previous formulations of $\mathit{ABK}$ for learning SBCs have issues with expressing appropriate constraints since the provided predicates do not take the structure of problem instances into account.
%
%
In this paper, we propose a new version of $\mathit{ABK}$ comprising two new types of auxiliary predicates. 
The first type encodes local properties of nodes in an input graph, while the second enables a more efficient constraint representation. That is, for 
two different nodes \texttt{N} and \texttt{M}, the predicate \texttt{close(N,M)} holds if these nodes share a common neighbor.
In case of bipartite graphs containing two types of nodes, \texttt{A} and \texttt{B}
(standing for sensors and zones in case of PUP instances),
we define two versions of this auxiliary predicate, distinguishing the two types by \texttt{closeA/2} and \texttt{closeB/2}. 
The second auxiliary predicate introduces an ordering on value assignments as follows: 
let $[1..\mathit{max}_x]$ and $[1..\mathit{max}_y]$ be the domains of two variables \texttt{X} and \texttt{Y}, 
and \texttt{p(X,Y)} be a binary predicate that holds
for at most one value $\text{\texttt{X}} \in [1..\mathit{max}_x]$
per $\text{\texttt{Y}} \in [1..\mathit{max}_y]$ in each answer set. 
%
Then, we define the following auxiliary predicate:
\begin{lstlisting}
  pGEQ(X,Y) :- p(X,Y).
  pGEQ(X,Y) :- pGEQ(X+1,Y), 0 < X.
\end{lstlisting}
If \texttt{pGEQ(X,Y)} is true, we know that \texttt{p(X',Y)} holds for some value \texttt{X'} equal to \texttt{X} or greater.
A constraint may then contain \texttt{pGEQ(Y,Y)} instead of the equivalent test \texttt{p(X,Y),~Y <=~X}, thus reducing the ground instantiation size. Moreover, this encoding can bring benefits for solving as well, as we obtain a more powerful propagation \citep{crabak94a}.

From a technical perspective, the framework presented by \cite{tagesc21a} runs \textsc{sbass} on a ground program resulting from the union of $P$, an instance $i \in S$, and $\mathit{ABK}$.
The inclusion of $\mathit{ABK}$ caused no difference in the symmetries for their approach (without iterative constraint learning),
given that the introduced auxiliary predicates do not affect atoms occurring in $P$. 
On the other hand, the predicate \texttt{pGEQ/2} alters the identification of symmetries for atoms over the predicate \texttt{p/2} contained in $P$.
Hence, we do not necessitate $\mathit{ABK}$ to contribute to a ground program passed to \textsc{sbass}, as indicated in Fig.~\ref{fig:pipeline}.
	
\paragraph{Language Bias.}
To address difficult combinatorial problems, we suggest the following set of mode declarations to define the search space $H_M$:
\begin{lstlisting}
  #modeb(1,r(var(t),var(t))).
  #modeb(1,close(var(t),var(t)),(symmetric,anti_reflexive)).
  #modeb(2,pGEQ(var(t),var(t))).
  #modeb(1,q(var(t),var(t))).
\end{lstlisting}
Assuming that the predicate \texttt{r/2} specifies the graph provided by an instance, the mode declaration in the first line expresses that one
such atom can occur per constraint.
Then, for each \texttt{close/2} or \texttt{pGEQ/2} predicate in $\mathit{ABK}$, we include a respective mode declaration as in the second and third lines.
Moreover, if atoms over another predicate \texttt{q/2} occur in~
$P$, but not in $\mathit{ABK}$, 
we supply a mode declaration of the last kind,
where considering binary predicates is sufficient for PUP instances.
Let us notice that only one type \texttt{t} of variables is used for all mode declarations. In this way, there is no restriction on the variable replacements, and we may explore inherent yet hidden properties of the input labels. 
Furthermore, we introduce an alternative scoring function assigning the cost $c_{\mathit{id}}$ to each rule $r_{\mathit{id}} \in H_M$. 
For every literal obtainable from 
the mode declarations, we overwrite its default cost~$1$ with a custom cost, except for literals over domain predicates like, e.g., \texttt{r/2}.
For literals over other predicates, in case the same variable is used for both of the contained arguments, the cost is $2$, and $3$ otherwise.
For example, the cost of the constraint
\texttt{:- pGEQ(V1,V1), close(V1,V2), q(V2,V3).}
would be equal to $1+1+1=3$ with the default scoring function but $2+1+3=6$ with our custom costs.
\subsection{Conflict Analysis}

As discussed in Section~\ref{subsec:cdilp}, \textsc{ilasp}'s Conflict Driven ILP (CDILP) approach requires a \textit{conflict analysis} method that, given a hypothesis $H$ and an example $e$ that $H$ does not cover, returns a coverage formula $F$.
For \textsc{ilasp} to function correctly, this formula must 
\begin{enumerate*}[label=\emph{(\roman*)}]
  \item not be satisfied by $H$, and 
  \item be satisfied by every hypothesis $H'\subseteq H_M$ 
covering $e$.
\end{enumerate*}

\textsc{Ilasp} has several built-in methods for conflict analysis, some of which are
described by \cite{law21a}. However, for learning tasks with hypothesis spaces
that consist of constraints only, these methods all behave equivalently when processing
positive examples. For positive examples $e=\langle e_{\mathit{pi}}, C\rangle$ such that
$B\cup C$ has many answer sets, the built-in methods can return extremely
long coverage formulas that also take long time to compute. For this reason, we
define a new conflict analysis method leading to shorter coverage
constraints that can be computed much faster.

\begin{definition}
  Let $T = \langle B, E, H_M\rangle$ be a learning task such that $H_M$
  consists of constraints, and let $e \in E^+$ be a positive example. 
  For any
  $H \subseteq H_M$ that does not cover~$e$, the
  \textit{subsumption-based conflict analysis} method $\mathit{sbca}(e, H, T)$
  returns the formula $\bigvee\limits_{r\in H,r'\in H_M}\bigwedge\limits_{r'
  \subseteq_{\theta} r} \lnot r'_{\mathit{id}}$,
  where $r'\subseteq_{\theta} r$ denotes that $r'$ subsumes~$r$.
\end{definition}

\begin{example}

  Consider a scenario such that \textsc{ilasp} is run on a task $T$ with the language
  bias given in the previous subsection. At some point in the execution, $T$
  may generate the hypothesis $H = \lbrace \texttt{:- close(V1,V2).}\;\;
  \texttt{:- not pGEQ(V1,V1), q(V1,V1).}\rbrace$. Within the hypothesis space
  computed by \textsc{ilasp},\footnote{This space is smaller than the full hypothesis
  space as isomorphic rules are discarded. For instance, \texttt{:- q(V2, V2).}
  is isomorphic to \texttt{:- q(V1, V1).} and thus not considered by ILASP.} the
  first rule is only subsumed by itself, and the second rule is
  subsumed by the following rules:

  {\small
  \begin{verbatim}
:- q(V1,V1).                      :- not pGEQ(V1,V1), q(V1,V1).
:- q(V1,V2).                      :- not pGEQ(V1,V1), not pGEQ(V1,V2), q(V1,V2).
:- pGEQ(V1,V1).                   :- not pGEQ(V1,V1), not pGEQ(V2,V1), q(V1,V2).
:- pGEQ(V1,V2).                   :- not pGEQ(V1,V1), not pGEQ(V2,V2), q(V1,V2).
:- not pGEQ(V1,V1), q(V1,V2).     :- not pGEQ(V1,V2), not pGEQ(V2,V1), q(V1,V2).
:- not pGEQ(V1,V2), q(V1,V2).     :- not pGEQ(V1,V2), not pGEQ(V2,V2), q(V1,V2).
:- not pGEQ(V2,V1), q(V1,V2).     :- not pGEQ(V2,V1), not pGEQ(V2,V2), q(V1,V2).
:- not pGEQ(V2,V2), q(V1,V2).
  \end{verbatim}
  }

  \noindent
  Let $r^1$ be the first rule in~$H$ and $r^2,\ldots, r^{16}$ be the rules that
  subsume the second rule. In this case, for any positive example $e$
  that is not covered by $H$, we obtain the coverage constraint $\mathit{sbca}(e, H, T) = (\lnot r^1_{\mathit{id}}) \lor
  ((\lnot r^2_{\mathit{id}})\land\ldots\land (\lnot r^{16}_{\mathit{id}}))$.
\end{example}

The following theorem shows that $\mathit{sbca}$ is a valid method for conflict
analysis for positive examples, provided that the hypothesis space contains
constraints only. This means that in our application domain, when using 
the $\mathit{sbca}$ method for
some or all of the positive examples, \textsc{ilasp} is guaranteed to return an
optimal solution for any learning task.

\begin{theorem}
  Let $T = \langle B, E, H_M\rangle$ be a learning task such that $H_M$
  consists of constraints, and let $e \in E^+$ be a positive example. 
  For any
  $H\subseteq H_M$ that does not cover $e$:
  \begin{enumerate}[leftmargin=*,align=left]
    \item
      $H$ does not satisfy $\mathit{sbca}(e, H, T)$.
    \item
      Every $H'\subseteq H_M$ that covers $e$ satisfies $\mathit{sbca}(e, H, T)$.
  \end{enumerate}
\end{theorem}

\begin{proof}
  \begin{enumerate}[leftmargin=*,align=left]
    \item
      For each $r \in H$, $r\subseteq_{\theta} r$
      implies that 
      $\bigwedge\limits_{r' \subseteq_{\theta} r} \lnot r'_{\mathit{id}}$
      is not satisfied by~$H$.
      Thus, $H$ cannot satisfy any of the disjuncts of
      $\mathit{sbca}(e, H, T)$, i.e., it does not satisfy
      $\mathit{sbca}(e, H, T)$.
    \item
      Assume for contradiction that some $H'\subseteq H_M$ covers $e = \langle
      e_{\mathit{pi}}, C\rangle$ but does not satisfy $\mathit{sbca}(e, H, T)$. Then, by
      the definition of covering $e$, there is some answer set $\mathcal{I} \in \mathit{AS}(B\cup
      C\cup H')$ that extends $e_{\mathit{pi}}$. As $H'$ does not satisfy
      $\mathit{sbca}(e, H, T)$, $H'$ must include constraints that
      subsume each of the constraints in~$H$. This means that
      $\mathit{AS}(B\cup C\cup H') \subseteq \mathit{AS}(B \cup C \cup H)$. Hence, $\mathcal{I} \in AS(B \cup C \cup H)$ contradicts the condition that $H$ does not cover $e$.
  \end{enumerate}
\end{proof}

The advantage of the $\mathit{sbca}$ method for conflict analysis, over the
methods that are already built-in to \textsc{ilasp}, is that this method does
not need to compute answer sets of $B \cup C \cup H$ (for an uncovered example
with context $C$). In other words, it does not need to consider the semantics
of the current hypothesis $H$ and can instead focus on purely syntactic
properties. In combinatorial problem domains, where finding answer sets can
be computationally intensive, our preliminary experiments showed that
$\mathit{sbca}$ is much faster than the existing conflict analysis methods of
\textsc{ilasp}.  On the other hand, the syntactic coverage constraints generated
by $\mathit{sbca}$ tend to be more specific than the semantic constraints
computed by \textsc{ilasp}'s built-in methods. That is, the formulas determined
by $\mathit{sbca}$ apply to fewer hypotheses and cut out fewer solution candidates, which
in turn means that more (yet considerably faster) iterations of the CDILP
procedure are required.

The $\mathit{sbca}$ method is closely related to the ILP system \textsc{Popper}
\citep{popper}, which also identifies syntactically determined constraints based on
subsumption. Specifically, when \textsc{Popper} encounters a hypothesis $H$
entailing an atom that should not be entailed (a negative example for
\textsc{Popper}), all hypotheses subsuming $H$ are discarded (as these would also
entail the atom). Just like $\mathit{sbca}$, since \textsc{Popper} uses
syntactically determined constraints, it can compute these very quickly but
may need many more iterations (compared to the
semantic conflict analysis methods of \textsc{Ilasp}).  While both
approaches are closely linked, \textsc{Popper} learns under Prolog semantics.
Therefore, \textsc{Popper} cannot reason about logic programs with multiple answer sets and
is inapplicable for learning ASP constraints in the combinatorial problem domains we address.

\paragraph{PyLASP Script.}
The $\mathit{sbca}$ method is essential to compute 
coverage formulas for positive examples obtained from the generalization set $\mathit{Gen}$,
including instances with a large number of answer sets,
in acceptable time. However, for 
other examples, the existing conflict analysis methods 
of \textsc{ilasp} are better suited, as the obtained formulas are more informative.
Hence, we extended the default \textsc{PyLASP} script of \textsc{ilasp}
with means to specify examples requiring $\mathit{sbca}$ usage by
dedicated identifiers, associated with instances in $\mathit{Gen}$ 
on which \textsc{clingo} takes more than $5$ seconds for enumerating the answer sets.
%

\section{Experiments}\label{sec:experiments}
For testing our new method, we decided to use 
PUP configuration benchmarks \citep{DBLP:journals/jcss/TeppanFG16}.
We applied our framework to PUP instances supplied by \cite{DBLP:conf/cpaior/AschingerDFGJRT11}, studying the \textit{double}, 
\textit{doublev}, 
and \textit{triple} 
instance collections
with $\ucap=\iucap=2$. Instances of the same type represent buildings of similar topology with scaling parameters that follow a common distribution. Although the benchmark instances are synthetic, they represent a relevant configuration problem concerning safety and security issues in public buildings, like administration offices or museums. In addition, the scalable synthetic benchmarks are easy to generate and analyze. That is, the nodes corresponding to rooms are labeled in a specific order, and the topologies follow a clear pattern.
The PUP instances and further details are provided by \cite{ilpsbc}.

In all three experiments, we learn first-order constraints using the same inputs to the ILP task, except for the two instance sets $S$ and $\mathit{Gen}$: for $S$, we pick the smallest representative instance of the selected type, while $\mathit{Gen}$ comprises the three smallest satisfiable instances other than the one in $S$.
Components of the ILP task shared between the experiments with different kinds of instances include:
\begin{enumerate*}[label=\emph{(\roman*)}]
  \item the input program $P$ as an ASP encoding of PUP that comprises no SBCs and originates from work by \cite{dogalemurisc16a}, where it is referred to as \texttt{ENC1};
  \item the background knowledge $\mathit{ABK}$ defining the auxiliary predicates \texttt{closesensors/2}, \texttt{closezones/2}, \texttt{unit2zoneGEQ/2}, and \texttt{unit2sensorGEQ/2}; and
  \item the language bias $M$ with mode declarations for the auxiliary predicates in $\mathit{ABK}$ as well as the predicates 
  \texttt{zone2sensor/2} and
  \texttt{partnerunits/2}, 
  the latter taking the roles of \texttt{r/2} and \texttt{q/2}
  according to the scheme described in Section~\ref{subsec:ilpTask}.
\end{enumerate*}

For each PUP instance type, we tested the two proposed approaches (\textit{scalable enum} and \textit{scalable fullSBCs}) combined with two versions of the scoring function: 
the default function of \textsc{ilasp}
as well as the custom costs introduced in Section~\ref{subsec:ilpTask}.
With both example generation approaches, the sampling of answer sets was done using the \texttt{--seed <}$seed$\texttt{>} option of \textsc{clingo} (v5.4.0). We executed the experiments using $120$ different random seeds to counterbalance the impact of randomness and get more reliable estimates of the relative learning and solving performance.
For comparing the two approaches, we sampled the same number of examples configured by $n$ for \textit{scalable enum} and $\text{\textit{cells}} \times \text{\textit{max\_cell\_size}}$ for \textit{scalable fullSBCs}. 
We limited the learning time to one hour, interrupting the process in case no positive or negative examples could be obtained from the analysis by \textsc{sbass}.

In preliminary investigations, we tried the default \textsc{PyLASP} script of \textsc{ilasp} (v4.1.2) as well as the one we devised to apply the $\mathit{sbca}$ method for conflict analysis to positive examples from $\mathit{Gen}$. As expected, no run with the default script could be finished within one hour, and we thus focus below on results obtained with our new \textsc{PyLASP} script.

The learning efficiency for all three PUP instance types is such that \textit{scalable fullSBCs} yields more successful ILP tasks, i.e., tasks solved by \textsc{ilasp} within the time limit, than \textit{scalable enum}.%
\footnote{The results reflect the learning setting leading to the fastest runtime, using the alternative ordering \textit{ord} but no \textit{sat} mode \citep{tagesc22}. Detailed records are provided by \cite{ilpsbc}.}
Using the \textit{scalable fullSBCs} strategy and $120$ different random seeds, we were able to finish \textsc{ilasp} runs with $108$ seeds for \textit{double}, $88$ for \textit{doublev}, and $12$ for \textit{triple}. The application of \textit{scalable enum} 
resulted in only $10$, $18$, or $6$ successful runs, respectively, with some of the $120$ seeds.  
%
Note that several \textit{scalable enum} runs had to be canceled because the generated 
ILP tasks were partial, i.e., without either positive or negative examples,
so that trivial optimal hypotheses make learning with \textsc{ilasp} obsolete. 
As \textit{scalable fullSBCs} is the by far more successful example generation strategy,
we restrict the following considerations of solving performance to constraints learned with it.

\begin{figure}[tb]
  \minipage{0.3\textwidth}
    \includegraphics[width=\linewidth]{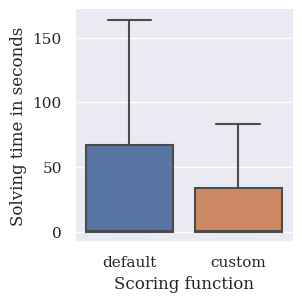}
    {\scriptsize
    \begin{align*} 
    &\text{Avg}: &70.89& &-&  &64.12\\
    &\text{Std}: &120.72& &-& &116.53\\
    &\text{TO}: &320(1568)& &-& &288(1568)
      \end{align*}
       }%
      \centering {\small \emph{(a) Results for PUP double}}
  \endminipage\hfill
  \minipage{0.3\textwidth}
    \includegraphics[width=\linewidth]{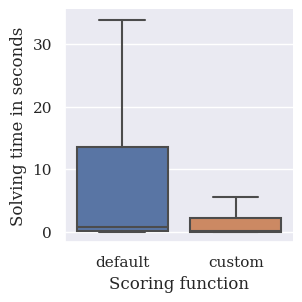}
    {\scriptsize
    \begin{align*} 
    &\text{Avg}: &43.11& &-&  &25.15\\
    &\text{Std}: &94.33& &-& &75.00\\
    &\text{TO}: &97(936)& &-& &55(936)
      \end{align*}
    }%
  \centering {\small \emph{(b) Results for PUP doublev}}
  \endminipage\hfill
  \minipage{0.3\textwidth}%
    \includegraphics[width=\linewidth]{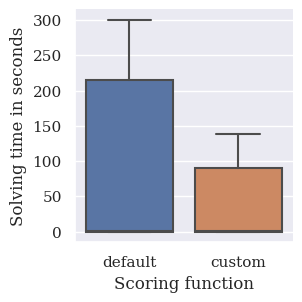}
    {\scriptsize
    \begin{align*} 
    &\text{Avg}: &81.56& &-&  &70.32\\
    &\text{Std}: &129.79& &-& &121.35\\
    &\text{TO}: &16(63)& &-& &13(63)
      \end{align*}
  }%
  \centering {\small \emph{(c) Results for PUP triple}}
  \endminipage
  \caption{Aggregated solving times for constraints learned with the two scoring functions}\label{fig:boxes}
  \end{figure}
Next, we compare the solving performance of \textsc{clingo} relative to
constraints learned with the default scoring function of \textsc{ilasp}
or the custom costs distinguishing domain predicates and variable recurrences as described in Section~\ref{subsec:ilpTask}.
To this end, we consider those of the $120$ random seeds for which
\textsc{ilasp} runs finished successfully with both of the scoring functions, and then provide the learned constraints as background
knowledge to \textsc{clingo} along with a PUP instance.
Each \textsc{clingo} run is limited to $300$ seconds solving time,
and Fig.~\ref{fig:boxes} displays box plots as well as average
runtimes (Avg), standard deviation (Std), and number of timeouts with
the total number of runs in parentheses (TO) for the 
\textit{double}, \textit{doublev}, and \textit{triple} instances.
For all three PUP instance types and especially the hard instances
whose runtime is above average, we observe significant speed-ups
due to constraints learned by means of the custom scoring function.
In fact, the custom costs give preference to first-order constraints
whose ground instances apply and prune the search space more directly,
thus benefiting the solving performance of \textsc{clingo}.
We also applied the Wilcoxon Signed-Rank test \citep{wilcoxon} for non-normally distributed runtimes, which confirms that the observed differences are statistically significant.
%

Finally, we contrast the best-performing learning setting, i.e.,
the \textit{scalable fullSBCs} strategy for example generation along with our custom
scoring function for assigning costs to constraints,
with originally proposed ASP encodings of PUP and instance-specific symmetry breaking.
In detail, Tables \ref{table:double}--\ref{table:triple} show runtime results for
the following systems and encodings:
%
%
%
\begin{enumerate*}[label=\emph{(\roman*)}]
  \item \textsc{clingo} on the plain \texttt{ENC1} encoding;
  \item \textsc{clingo} on \texttt{ENC1} augmented with the most efficient
        learned constraints among those aggregated in Fig.~\ref{fig:boxes} as $\mathit{ABK}$;
  \item \textsc{clingo} on the advanced \texttt{ENC2} encoding \citep{dogalemurisc16a},
        incorporating hand-crafted static symmetry breaking as well as an ordered
        representation \citep{crabak94a} of assigned units similar to
        \texttt{pGEQ/2} in Section~\ref{subsec:ilpTask};
  \item \textsc{sbass} for permutation and ground SBCs computation on \texttt{ENC1}; and 
  \item \textsc{clasp}$^\pi$ denoting the solving time of \textsc{clingo} on \texttt{ENC1}
        augmented with ground SBCs by \textsc{sbass}.%
        \footnote{In the online usage of instance-specific symmetry breaking,
                  the runtimes of \textsc{sbass} and \textsc{clasp}$^\pi$ add up.}
\end{enumerate*}
%
PUP instances are named according to the scheme
\textbf{[un-]type-zones}, where \textbf{un}
indicates unsatisfiability due to including one unit less than required,
\textbf{type} denotes the
\textit{double}, \textit{doublev}, and \textit{triple} collections
by dbl, dblv, or tri,
and the number of \textbf{zones} provides a measure of size.
Each run is limited to $600$ seconds, and the TO entries mark
unfinished runs.

Regarding different PUP encodings, \texttt{ENC2} leads
to more robust \textsc{clingo} performance than the simpler \texttt{ENC1}
encoding, even if ground SBCs from \textsc{sbass} are included for
\textsc{clasp}$^\pi$.
That is, apart from a few shorter runs with \texttt{ENC1}
on satisfiable instances (dbl\nobreakdash-30, dblv\nobreakdash-45, tri\nobreakdash-18, and tri\nobreakdash-21) in
Tables \ref{table:double}--\ref{table:triple},
\textsc{clingo} scales better with \texttt{ENC2}, never times out on
instances finished with \texttt{ENC1} or possibly \textsc{clasp}$^\pi$, and is able to solve some instances (dbl\nobreakdash-50, un\nobreakdash-tri\nobreakdash-12, and un\nobreakdash-tri\nobreakdash-15)
on which the latter two settings fail.
Considering that \textsc{sbass} produces ground SBCs within the
time limit for all instances,
the better performance with \texttt{ENC2} suggests that its
hand-crafted static symmetry breaking approach provides a more economic
trade-off between the compactness and completeness of introduced SBCs.
However, we checked that static symmetry breaking by unit labels
counteracts \textsc{sbass} to detect any instance-specific symmetries, so that
\texttt{ENC2} commits to value symmetries only.
\begin{table}[tb]
  \begin{minipage}{.5\linewidth}
      \centering
      \rowcolors{2}{gray!25}{white}
      \setlength{\tabcolsep}{7.5pt}
      \resizebox{6.5cm}{!}{
      \csvloop{
      file=PUP_double.csv,
      head to column names,
      before reading=\centering\sisetup{table-number-alignment=center},
      tabular={lrrrrr},
      table head=\toprule & \textbf{\textsc{ENC1}} & \textbf{ABK} & \textbf{ENC2} &  \textbf{SBASS} & $\mathbf{CLASP^\pi}$\\\midrule,
      command=\Instance & \BASE & \ABK & \ENC & \SBASS & \CLASP,
      table foot=\bottomrule}}
      \caption{Runtimes for PUP double}
      \label{table:double}
  \end{minipage}%
  \begin{minipage}{.5\linewidth}
      \centering
      \rowcolors{2}{gray!25}{white}
      \setlength{\tabcolsep}{7.5pt}
      \resizebox{6.5cm}{!}{
      \csvloop{
      file=PUP_doublev.csv,
      head to column names,
      before reading=\centering\sisetup{table-number-alignment=center},
      tabular={lrrrrr},
      table head=\toprule & \textbf{\textsc{ENC1}} & \textbf{ABK} & \textbf{ENC2} &  \textbf{SBASS} & $\mathbf{CLASP^\pi}$\\\midrule,
      command=\Instance & \BASE & \ABK & \ENC & \SBASS & \CLASP,
      table foot=\bottomrule}}
      \caption{Runtimes for PUP doublev}
      \label{table:doublev}
  \end{minipage} 
\end{table}
\begin{table}[tb]
  \centering
  \rowcolors{2}{gray!25}{white}
  \setlength{\tabcolsep}{7.5pt}
  \resizebox{7cm}{!}{
  \csvloop{
  file=PUP_triple.csv,
  head to column names,
  before reading=\centering\sisetup{table-number-alignment=center},
  tabular={lrrrrr},
  table head=\toprule & \textbf{\textsc{ENC1}} & \textbf{ABK} & \textbf{ENC2} &  \textbf{SBASS} & $\mathbf{CLASP^\pi}$\\\midrule,
  command=\Instance & \BASE & \ABK & \ENC & \SBASS & \CLASP,
  table foot=\bottomrule}}
  \caption{Runtimes for PUP triple}
  \label{table:triple}
\end{table}%

Indeed, when comparing \texttt{ENC2} to \textsc{clingo} on \texttt{ENC1} with 
first-order constraints learned by \textsc{ilasp} as $\mathit{ABK}$,
we observe further significant performance improvements thanks
to our approach, particularly on the unsatisfiable instances in
Tables \ref{table:double}--\ref{table:triple}.
That is, the learned $\mathit{ABK}$ enables \textsc{clingo}
to solve the considered PUP instances of three different types and
efficiently prunes the search space, which must be fully explored in case of unsatisfiability.
%
%
We checked that the learned constraints exploit the topology of instances to restrict the assignable units, particularly using the predicate \texttt{unit2sensorGEQ/2} for referring to the assignment of sensors.
The specific sets of efficient first-order constraints learned for each of the three PUP instance types are provided in our repository \citep{ilpsbc}.

\section{Conclusions}\label{sec:conclusions}
This paper introduces an approach to learn first-order constraints for complex combinatorial problems, 
which cannot be successfully tackled by previous, less scalable methods \citep{tagesc22}.
We devised and implemented two configurable strategies to generate the 
examples for an ILP task,
and extended the background knowledge by auxiliary predicates admitting a more compact
representation and potentially stronger propagation of constraints on value assignments.
Moreover, we introduced a custom scoring function taking the ground instances of
first-order constraints into account, along with a new conflict analysis method for \textsc{ilasp} that enables a much faster handling of positive examples with many answer sets.
The revised learning framework taking all proposed techniques together
is able to learn efficient first-order constraints from non-trivial
problem instances, as demonstrated on three kinds of PUP benchmarks --
a challenging configuration problem.
%
In the future, we hope 
to introduce automatic (re-)labeling schemes for constants appearing in instances to exploit common problem structure in a less input-specific way.
Moreover, we aim at extending 
our learning framework further to enable the model-based analysis and
breaking of symmetries for practically important optimization problems.
Beyond graph-oriented applications,
we plan to expand the scope of our constraint learning methods to further areas, such as scheduling domains with symmetries among tasks and resources, e.g., several instances of products, machines, or workers with the same skills.

\paragraph{Acknowledgments.}{This work was partially funded by
KWF project 28472,
cms electronics GmbH,
FunderMax GmbH,
Hirsch Armbänder GmbH,
incubed IT GmbH,
Infineon Technologies Austria AG,
Isovolta AG,
Kostwein Holding GmbH, and
Privatstiftung Kärntner Sparkasse.
We thank the anonymous reviewers for their valuable suggestions and comments.}


\end{document}